\documentclass[11pt]{myclass}
\usepackage{color}
\usepackage{paralist} 
\usepackage{nicefrac} 
\usepackage{subfigure}
\usepackage{natbib}
\usepackage{amsmath,amssymb}
\usepackage{xspace}
\usepackage{graphicx}
\usepackage{euscript}
\usepackage{wrapfig}
\usepackage{algorithm,algorithmic}


\newcommand{\delete}[1] {}

\newcommand{\eps}{\varepsilon}

\newcommand{\done}{\textit{Synthetic1}\xspace}
\newcommand{\dtwo}{\textit{Synthetic2}\xspace}
\newcommand{\dthr}{\textit{Synthetic3}\xspace}
\newcommand{\dfour}{\textit{Synthetic4}\xspace}
\newcommand{\dfive}{\textit{Synthetic5}\xspace}
\newcommand{\dsix}{\textit{Synthetic6}\xspace}

\newcommand{\mush}{\textit{Mushroom}\xspace}

\newcommand{\cancer}{\textit{Cancer}\xspace}

\newcommand{\naiv}{\textsc{Naive}\xspace}
\newcommand{\vote}{\textsc{Voting}\xspace}
\newcommand{\randorg}{\textsc{Rand}\xspace}
\newcommand{\rand}{\textsc{RandEmp}\xspace}
\newcommand{\supo}{\textsc{MaxMarg}\xspace}
\newcommand{\med}{\textsc{Median}\xspace}

\newcommand{\itsupp}{\textsc{IterativeSupports}\xspace}

\newcommand{\ritsupp}{\textsc{WeightedSampling}\xspace}
\newcommand{\oalgo}{\textsc{Mwu}\xspace}
\newcommand{\oalgoes}{\textsc{MwuEmp}\xspace}

\newcommand{\denselist}{\itemsep -2pt\parsep=-1pt\partopsep -2pt \vspace{-.05in}}

\theoremstyle{definition}

\begin{document}

\title{Efficient Protocols for Distributed Classification and Optimization\thanks{This work was sponsored in part by the NSF grants CCF-0953066 and CCF-0841185 and in part by the DARPA CSSG grant N11AP20022. All the authors gratefully acknowledge the support of the grants. Any opinions, findings, and conclusion or recommendation expressed in this material are those of the author(s) and do not necessarily reflect the view of the funding agencies or the U.S. government.}}

\author{Hal Daum\'e III \\ University of Maryland, College Park \and Jeff M. Phillips \\ University of Utah \and Avishek Saha \\ University of Utah \and Suresh Venkatasubramanian \\ University of Utah }



\maketitle

\begin{abstract}
In distributed learning, the goal is to perform a learning task over data distributed across multiple nodes with minimal (expensive) communication. Prior work~\citep{daume12distributed} proposes a general model that bounds the communication required for learning classifiers while allowing for $\eps$ training error on linearly separable data adversarially distributed across nodes. 

In this work, we develop key improvements and extensions to this basic model. Our first result is a two-party \emph{multiplicative-weight-update} based protocol that uses $O(d^2 \log{1/\eps})$ words of communication to classify distributed data in arbitrary dimension $d$,  $\eps$-optimally. This readily extends to classification over $k$ nodes with $O(kd^2 \log{1/\eps})$ words of communication. 
Our proposed protocol is simple to implement and is considerably more efficient than baselines compared, as demonstrated by our empirical results. 

In addition, we illustrate general algorithm design paradigms for doing efficient learning over distributed data. We show how to solve fixed-dimensional and high dimensional linear programming efficiently in a distributed setting where constraints may be distributed across nodes. Since many learning problems can be viewed as convex optimization problems where constraints are generated by individual points, this models many typical distributed learning scenarios. Our techniques make use of a novel connection from multipass streaming, as well as adapting the multiplicative-weight-update framework more generally to a distributed setting. As a consequence, our methods extend to the wide range of problems solvable using these techniques. 
\end{abstract}


\section{Introduction}
In recent years, distributed learning (learning from data spread across multiple locations) has witnessed a lot of research interest~\citep{langford11dist}. One of the major challenges in distributed learning is to minimize communication overhead between different parties, each possessing a disjoint subset of the data.  Recent work~\citep{daume12distributed} has proposed a distributed learning model that seeks to minimize communication by carefully choosing the most informative data points at each node. The authors present a number of general sampling based results as well as a specific two-way protocol that provides a logarithmic bound on communication for the family of linear classifiers in $\mathbb{R}^2$. Most of their results pertain to two players but they propose basic extensions for multi-player scenarios. A distinguishing feature of this model is that it is \emph{adversarial}. Except linear separability, no distributional or other assumptions are made on the data or how it is distributed across nodes. 

In this paper, we develop this model in two substantial ways. First, we extend the results on linear classification to arbitrary dimensions, in the process presenting a more general algorithm that does not rely on explicit geometric constructions. This approach exploits the multiplicative weight update (MWU) framework (specifically its use in boosting) and retains desirable theoretical guarantees -- \emph{data-size-independent} communication between nodes in order to classify data -- while being simple to implement. Moreover, it easily extends to $k$-players with an additional  $k$ communication over the two-player result, which improves the earlier results in two dimensions by a factor of $k$. A second contribution of this work is to demonstrate how general convex optimization problems (for example, linear programming, SDPs and the like) can be solved efficiently in this distributed framework using ideas from both multipass streaming, as well as the well-known multiplicative weight update method. Since many (batch) learning tasks can be reduced to convex optimization problems, this second contribution opens the door to deploying many other learning tasks in the distributed setting with minimal communication. 

\paragraph{Outline.} 
Our main two-party result is proved in Section~\ref{sec:r-2party}, based on background in Section~\ref{sec:background}. Using a new sampling protocol for $k$ players (Section~\ref{sec:improved-random}) we extend the two-party result to $k$ players in Section~\ref{sec:k-party-protocol} and present an empirical study in Section~\ref{sec:experiments}. In Section~\ref{sec:opt} we present our results for distributed optimization. 

\paragraph{Related Work.}
Existing work in distributed learning mainly focuses on either inferring an accurate global classifier from multiple distributed sub-classifiers learned individually (at respective nodes) or on improving the efficiency of the overall learning protocol. The first line of work consists of techniques like \emph{parameter mixing}~\citep{mcdonald10diststrucperc,mann09distmem} or \emph{averaging}~\citep{collins02discHMM} and classifier \emph{voting}~\citep{bauer99voting-emp}. These approaches do admit convergence results but lack any bounds on the communication. Voting, on the other hand, has been shown~\citep{daume12distributed} to yield suboptimal results on adversarially partitioned datasets. The goal of the second line of work is to make distributed algorithms scale to very large datasets; many of these works~\citep{NIPS2006_725,Teo:2010:BMR:1756006.1756016} depend on MapReduce to extract performance improvement. \cite{DBLP:journals/corr/abs-1012-1367} averaged over mini-batches of accumulated gradients to improve regret bounds for distributed online settings. \citep{NIPS2010_1162} proposed a MapReduce based improved parallel stochastic gradient descent and more recently~\citep{NIPS2011_0771} improved the time complexity of $\gamma$-margin parallel algorithms from $\Omega(1/\gamma^2)$ to $O(1/\gamma)$. Finally,~\citep{NIPS2010_0423} and~\citep{NIPS2011_0574} consider optimization in distributed settings but their convergence analysis applies to specific cases of subgradient and stochastic gradient descent algorithms.

Surprisingly, communication in learning has not been studied as a resource to be used sparingly. And as \citep{daume12distributed} and this work demonstrates, intelligent interaction between nodes, communicating relevant aspects of the data, not just its classification, can greatly reduce the necessary communication over existing approaches. On large distributed systems, communication has become a major bottleneck for many real-world problems; it accounts for a large percentage of total energy costs, and is the main reason that MapReduce algorithms are designed to minimize rounds (of communication).  This strongly motivates the need to incorporate the study of this aspect of an algorithm directly, as presented and modeled in this paper.  

Recently but independently, research by \citep{balcan12colt} considers very similar models to those of \citep{daume12distributed}. They also consider adversarially distributed data among $k$ parties and attempt to learn on the adversarially distributed data while minimizing the total communication between the parties. Like \citep{daume12distributed} the work of \citep{balcan12colt} presents both agnostic and non-agnostic results for generic settings, and shows improvements over sampling bounds in several specific settings including the $d$-dimensional linear classifier problem we consider here (also drawing inspiration from boosting). In addition, their work provides total communication bounds for decision lists and for proper and non-proper learning of parity functions. They also extend the model so as to preserve differential and distributional privacy while conserving total communication, as a resource, during the learning process.  

In contrast, this work identifies optimization as a key primitive underlying many learning tasks, and focuses on solving the underlying optimization problems as a way to provide general communication-friendly distributed learning methods. We introduce techniques that rely on multiplicative weight updates and multi-pass streaming algorithms.  Our main contributions are translating these techniques into this distributed setting and using them to solve LPs (and SDPs) in addition to solving for $d$-dimensional linear separators.  

\section{Background}
\label{sec:background}
In this section, we revisit the model proposed in \citep{daume12distributed} and mention related results.

\paragraph{Model.}
We assume that there are $k$ parties $P_1, P_2, \ldots P_k$. Each party $P_i$ possesses a dataset $D_i$ that no other party has access to, and each $D_i$ may have both positive and negative examples.  The goal is to classify the full dataset $D = \cup_i D_i$ correctly.  
We assume that there exists a perfect classifier $h^*$ from a family of classifiers $\EuScript{H}$ with associated range space $(D,\EuScript{H})$ and bounded VC-dimension $\nu$.  We are willing to allow $\eps$-classification error on $D$ so that up to $\eps |D|$ points in total are misclassified.  

Each \emph{word} of data (e.g., a single point or vector in $\mathbb{R}^d$ counts as $O(d)$ words) passed between any pair of parties is counted towards the total communication; this measure in words allows us to examine the cost of extending to $d$-dimensions, and allows us to consider communication in forms other than example points, but does not hinder us with precision issues required when counting bits.  
For instance, a protocol that broadcasts a message of $M$ words (say $M/d$ points in $\mathbb{R}^d$) from one node  to the other $k-1$ players costs $O(kM)$ communication. 
The goal is to design a protocol with as little communication as possible. We assume an \emph{adversarial} model of data distribution; in this setting we prepare for the worst, and allow some \emph{adversary} to determine which player gets which subset of $D$.  

\paragraph{Sampling bounds.}
Given any dataset $D$ and a family of classifiers with bounded VC-dimension $\nu$, then a random sample of size 
\begin{equation}
s_{\eps,\nu} = O(\min\{ (\nu/\eps) \log (\nu/\eps), \nu/\eps^2 \})
\label{eq:sampling}
\end{equation}
from $D$ has at most $\eps$-classification error on $D$ with constant probability~\citep{book09anthony}, as long as there exists a perfect classifier.  Throughout this paper we will assume that a perfect classifier exists.  
This constant probability of success can be amplified to any $1-\delta$ with an extra $O(\log (1/\delta))$ factor of samples.

\paragraph{Randomly partitioned distributions.}
Assume that for all $i \in [1,k]$, each party $P_i$ has a dataset $D_i$ drawn from the same distribution.  That is, all datasets $D_i$ are identically distributed.  This case is much simpler than what the remainder of this paper will consider.  
Using (\ref{eq:sampling}), each $D_i$ can be viewed as a sample from the full set $D = \cup_i D_i$, and with \emph{no} communication each party $P_i$ can faithfully estimate a classifier with error $O((\nu/|D_i|) \log (\nu |D_i|))$.  

Henceforth we will focus on \emph{adversarially} distributed data.  

\paragraph{One-way protocols.}
Consider a restricted setting where protocols are only able to send data from parties $P_i$ (for $i\geq 2$) to $P_1$; a restricted form of \emph{one-way communication}.  
We can again use (\ref{eq:sampling}) so that all parties $P_i$ send a sample $S_i$ of size $s_{\eps,\nu}$ to $P_1$, and then $P_1$ constructs a global classifier on $\cup_{i=2}^k S_i$ with $\eps$-classification error $\cup_{i=1}^k D_i$; this requires $O(d k s_{\eps,\nu})$ words of communication for points in $\mathbb{R}^d$.  

For specific classifiers~\cite{daume12distributed} do better.  
For thresholds and intervals one can learn a \emph{zero}-error distributed classifier using constant amount of one-way communication. The same can be achieved for axis-aligned rectangles with $O(kd^2)$ words of communication.  However, those authors show that hyperplanes in $\mathbb{R}^d$, for $d\geq 2$, require at least $\Omega(k/\eps)$ one-way bits of communication to learn an $\eps$-error distributed classifier.

\paragraph{Two-way protocols.}
Hereafter, we consider two-way protocols where any two players can communicate back and forth.  
It has been shown \citep{daume12distributed} that, in $\mathbb{R}^2$, a protocol can learn linear classifiers with at most $\eps$-classification error using at most $O(k^2 \log{1/\eps})$ communication.  This protocol is deterministic and relies on a complicated pruning argument, whereby in each round, either an acceptable classifier is found, or a constant fraction more of some party's data is ensured to be classified correctly.


\section{Improved Random Sampling for $k$-players}
\label{sec:improved-random}
Our first contribution is an improved two-way $k$-player sampling-based protocol using \emph{two-way} communication and the sampling result in (\ref{eq:sampling}).  
We designate party $P_1$ as a coordinator, and it gathers the size of each player's dataset $D_i$, simulates sampling from each player completely at random, and then reports back to each player the number of samples to be drawn by it, in $O(k)$ communication.  Then each other party $P_i$ selects $s_{\eps,\nu} |D_i|/|D|$ random points (in expectation), and sends them to the coordinator.  The union of this set satisfies the conditions of the result from (\ref{eq:sampling}) over $D = \cup_i D_i$ and yields the following result.  

\begin{theorem} 
\label{thm:new-generickway}
Consider any family of hypothesis that has VC-dimension $\nu$ for points in $\mathbb{R}^d$. Then there exists a two-way $k$-player protocol using $O(kd + d\min\{(\nu/\eps) \log (\nu/\eps), \nu/\eps^2\})$ total words of communication that achieves $\eps$-classification error, with constant probability.  
\end{theorem}

Again using two-way communication, this type of result can be made even more general.  Consider the case where each $P_i$'s dataset arrives in a continuous stream; this is what is known as a \emph{distributed data stream}~\citep{DBLP:conf/soda/CormodeMY08}.  Then applying results of \citep{DBLP:conf/pods/CormodeMYZ10}, we can continually maintain a sufficient random sample at the coordinator of size $s_\eps$ communicating $O((k + s_{\eps,\nu}) d \log |D|)$ words.  

\begin{theorem} 
\label{thm:new-generickway-stream}
Consider any family of hypothesis that has VC-dimension $\nu$ for points in $\mathbb{R}^d$. 
Let each of $k$ parties have a stream of data points $D_i$ where $D = \cup_i D_i$.  
Then there exists a two-way $k$-player protocol using $O((k + \min\{(\nu/\eps) \log (\nu/\eps), \nu/\eps^2\})$ $d\log |D|)$ total words of communication that maintains $\eps$-classification error, with constant probability.  
\end{theorem}


\section{A Two-Party Protocol}
\label{sec:r-2party}

In this section, we consider only two parties, and for notational clarity, we refer to them as $A$ and $B$. $A's$ dataset is labeled $D_A$ and $B$'s dataset is labeled $D_B$.  Let $|D_B| = n$.  
Our protocol, summarized in Algorithm~\ref{alg:ouralgo}, is called \ritsupp. 
In each round, $A$ sends a classifier $h_A$ to $B$ and $B$ responds back with a set of points $R_B$, which it constructs by sampling from a weighting on its points. 
At the end of $T$ rounds (for $T = O(\log(1/\eps))$), we will show that by voting on the result from the set of $T$ classifiers $h_A$ will misclassify at most $\eps |D_B|$ points from $D_B$ while being perfect on $D_A$, and hence $\eps |D_B| < \eps |D_B \cup D_A| = \eps |D|$, yielding a $\eps$-optimal classifier as desired.  

There are two ways $R_B$ can construct its points: a random sample and a deterministic sample.  For simplicity, we will focus our presentation on the randomized version since it is more practical, although it has slightly worse bounds in the two-party case.  Then we will also mention and analyze the deterministic version.  

It remains to describe how $B$'s points are weighted and updated, which dictates how $B$ constructs the sample sent to $A$.
Initially, they are all given a weight $w_1 = 1$.  Then the re-weighting strategy (described in Algorithm \ref{alg:compsup}) is an instance of the multiplicative weight update framework; with each new proposed classifier $h_A$ from $A$, party $B$ increases all weights of misclassified points by a $(1+\rho)$ factor, and does not change the weight for correctly classified points.  We will show $\rho=0.75$ is sufficient.  
Intuitively, this ensures that consistently misclassified points eventually get weighted high enough that they are very likely to be chosen as examples to be communicated in future rounds.
The deterministic variant simply replaces Line 7 of Algorithm \ref{alg:compsup} with the weighted variant \citep{Mat91} of the deterministic construction of $R_B$ \citep{Cha01}; see details below.   

Note that this is roughly similar in spirit to the heuristic protocol  \citep{daume12distributed} that exchanged support points and was called \itsupp, which we will experimentally compare against.  But the protocol proposed here is less rigid, and as we will demonstrate next, this allows for a much less nuanced analysis.  

\begin{algorithm}[!htbp]
 \caption{\ritsupp}
 \begin{algorithmic}
   \STATE \textbf{Input:} $D_A, D_B$, parameters: $0 < \eps < 1$
   \STATE \textbf{Output:} $h_{AB}$ (classifier with $\eps$-error on $D_A \cup D_B$)
	 \STATE \textbf{Init:} $R_B=\{\}$; $w_i^0 = 1\ \forall x_i \in D_B$;        
   \FOR{t = 1 $\ldots$ $T = 5 \log_2(1/\eps)$} 
     \STATE --------- \textbf{A's move} ---------
		 \STATE $D_A = D_A \cup R_B$; 
		 \STATE $h_A^t := Learn(D_A)$; 
		 \STATE send $h_A^t$ to $B$;
     \STATE --------- \textbf{B's move} ---------
   	 \STATE $R_B$ := \oalgo($D_B$, $h_A^t$, $\rho = 0.75$, $c = 0.2$); send $R_B$ to $A$;
	 \ENDFOR
	 \STATE $h_{AB} = \textsf{Majority}(h_A^1,h_A^2,\ldots,h_A^T)$;
 \end{algorithmic}
 \label{alg:ouralgo}
\end{algorithm}

\begin{algorithm}[!htbp]
 \caption{\oalgo($D_B$, $h_A^t$, $\rho$, $c$)}
 \begin{algorithmic}[1]
   \STATE \textbf{Input:} $h_A^t, D_B$, parameters: $0 < \rho < 1$, $0 < c < 1$
   \STATE \textbf{Output:} $R_B$ (a set of $s_{c,d}$ points)
	 \FORALL{($x_i \in D_B$)}
	   \STATE if({$h_A^t(x_i) \neq y_i$}) then $w_i^{t+1} = w_i^t(1+\rho)$;
	   \STATE if({$h_A^t(x_i) == y_i$})   then $w_i^{t+1} = w_i^t$;
	 \ENDFOR
   \STATE randomly sample $R_B$ from $D_B$ (according to  $w^{t+1}$);
 \end{algorithmic}
 \label{alg:compsup}
\end{algorithm}

\subsection{Analysis}
Our analysis is based on the multiplicative weight update framework (and closely resembles boosting). First, we state a key structural lemma. Thereafter, we use this lemma to prove our main result. 

As mentioned above (see (\ref{eq:sampling})), after collecting a random sample $S_\eps$ of size $s_{\eps,d} = O(\min\{ (d/\eps) \log(d/\eps), d/\eps^2\})$ drawn over the entire dataset $D \subset \mathbb{R}^d$, a linear classifier learned on $S_\eps$ is sufficient to provide $\eps$-classification error on all of $D$ with constant probability.  There exist deterministic constructions for these samples $S_\eps$ still of size $s_{\eps,\nu}$ \citep{Cha01}; although they provide at most $\eps$-classification error with probability $1$, they, in general, run in time exponential in $\nu$.  
Note that the VC-dimension of linear classifiers in $\mathbb{R}^d$ is $O(d)$, and these results still holds when the points are weighted and the sample is drawn (respectively constructed \citep{Mat91}) and error measured with respect to this weighting distribution.    
Thus $B$ could send $s_{\eps,d}$ points to $A$, and we would be done; but this is too expensive.  We restate this result with a constant $c$, so that at most a $c$ fraction of the weights of points are mis-classified (later we show that $c=0.2$ is sufficient with our framework).  
Specifically, setting $\eps=c$ and rephrasing the above results yields the following lemma.  

\begin{lemma}
\label{lem:cnet}
Let $B$ have a weighted set of points $D_B$ with weight function $w : D_B \to \mathbb{R}^+$.  
For any constant $c>0$, party $B$ can send a set $S_{c,d}$ of size $O(d)$ (where the constant depends on $c$) such that any linear classifier that correctly classifies all points in $S_{c,d}$ will misclassify points in $D_B$ with a total weight at most $c \sum_{x \in D_B} w(x)$.
The set $S_{c,d}$ can be constructed deterministically, or a weighted random sample from $(D_B,w)$ succeeds with constant probability.  
\end{lemma}

We first state the bound using the deterministic construction of the set $S_{c,d}$, and then extend it to the more practical (from a runtime perspective) random sampling result, but with a slightly worse communication bound.  

\begin{theorem}
\label{thm:main-2party}
The deterministic version of two-party two-way protocol \ritsupp for linear separators in $\mathbb{R}^d$ misclassifies at most $\eps |D|$ points after $T = O(\log(1/\eps))$ rounds using $O(d^2\log(1/\eps))$ words of communication.
\end{theorem}


\begin{proof}
At the start of each round $t$, let $\phi_t$ be the potential function given by the sum of weights of all points in that round.  Initially, $\phi_1 = \sum_{x_i \in D_B} w_i = n$ since by definition for each point $x_i \in D_B$ we have $w_i = 1$.  

Then in each round, $A$ constructs a classifier $h_A^t$ at $B$ to correctly classify the set of points that accounts for at least $1-c$ fraction of the total weight by Lemma \ref{lem:cnet}. All other misclassified points are upweighted by $(1+\rho)$. Hence, for round $(t+1)$ we have
$	\phi^{t+1}  \leq \phi^t \left((1-c) + c(1+\rho) \right)  =  \phi^t \left( 1 + c \rho \right) = n \left( 1 + c \rho \right)^t$.

Let us consider the weight of the points in the set $S \subset D_B$ that have been misclassified by a majority of the $T$ classifiers (after the protocol ends).  This implies every point in $S$ has been misclassified \emph{at least} $T/2$ number of times and \emph{at most} $T$ number of times. So the minimum weight of points in $S$ is $(1+\rho)^{T/2}$ and the maximum weight is $(1+\rho)^T$.

Let $n_i$ be the number of points in $S$ that has weight $(1+\rho)^i$ where $i \in [T/2,T]$. The potential function value of $S$ after $T$ rounds is $\phi^T_S = \sum_{i=T/2}^T n_i (1+\rho)^i$. 
Our claim is that $\sum_{i=1}^T n_i  = |S| \leq \eps n$.  
Each of these at most $|S|$ points have a weight of at least $(1+\rho)^{T/2}$. Hence we have that 
\[
	\phi^T_S  = \sum_{i=T/2}^T n_i (1+\rho)^i  
						\geq (1+\rho)^{T/2} \sum_{i=T/2}^T n_i 
						= (1+\rho)^{T/2} |S|.
\]

Relating these two inequalities we obtain the following,
\[
 |S| (1+\rho)^{T/2} \leq \phi^T_S  \leq  \phi^T = n \left( 1 + c \rho \right)^T.
\]
Hence (using $T = 5\log_2(1/\eps)$)
\[
|S|  \leq  n  \left( \frac{(1+c \rho)}{( 1 + \rho)^{1/2}} \right)^T 
  = 
  n \left( \frac{(1+c\rho)}{\left( 1 + \rho \right)^{1/2}} \right)^{5 \log_2(1/\eps)}
  =
  n (1/\eps)^{5 \log_2 \left( \frac{( 1 + c\rho)}{(1+\rho)^{1/2}} \right)} .
\]

Setting $c = 0.2$ and $\rho = 0.75$ we get $5 \log_2 \left((1+c \rho)/(1+\rho)^{1/2}\right)) < -1$ and thus $|S| < n (1/\eps)^{-1} < \eps n$, as desired since $\eps < 1$.   
Thus each round uses $O(d)$ points, each requiring $d$ words of communication, yielding a total communication of $O(d^2 \log(1/\eps))$.
\end{proof}

In order to use random sampling (as suggested in Algorithm \ref{alg:compsup}), we need to address the probability of failure of our protocol.  That is, more specifically the set $S_{c,d}$ in Lemma \ref{lem:cnet} is of size $O(d \log (1/\delta'))$ and a linear classifier that has no error on $S_{c,d}$ misclassifies points in $D_B$ with weight at most $c \sum_{x \in D_B} w(x)$, with probability at least $1- \delta'$.  

However, we would like this probability of failure to be a constant $\delta$ over the entire course of the protocol.  To guarantee this, we need the $c$-misclassification property to hold in each of $T$ rounds.  Setting $\delta' = \delta/T$, and applying the union bound implies that then the probability of failure at any point in the protocol is at most $\sum_{i=1}^T \delta' = \sum_{i=1}^T \delta/T = \delta$.  This increases the communication cost of each round to $O(d^2 \log (1/\delta')) = O(d^2 \log (\log (1/\eps)/\delta)) = O(d^2 \log \log (1/\eps))$ words, with a constant $\delta$ probability of failure.  Hence using random sampling as described in \ritsupp requires a total of $O(d^2 \log (1/\eps) \log \log (1/\eps))$ words of communication.  We formalize below.

\begin{theorem}
\label{thm:main-2party-rand}
The randomized two-party two-way protocol \ritsupp for linear separators in $\mathbb{R}^d$ misclassifies at most $\eps |D|$ points, with constant probability, after $T = O(\log(1/\eps))$ rounds using $O(d^2\log(1/\eps) \log \log (1/\eps))$ words of communication.
\end{theorem}

\section{$k$-Party Protocol}
\label{sec:k-party-protocol}
In Section~\ref{sec:improved-random} we described a simple protocol (Theorem~\ref{thm:new-generickway}) to learn a classifier with $\eps$-error jointly among $k$ parties using $O(kd + d\min\{ \nu/\eps \log (\nu/\eps), \nu/\eps^2 \})$ words of total communication. We now combine this with the two-party protocol from Section~\ref{sec:r-2party} to obtain a $k$-player protocol for learning a joint classifier with error $\eps$. 

We fix an arbitrary node (say $P_1$) as the coordinator for the $k$-player protocol of Theorem~\ref{thm:new-generickway}. Then $P_1$ runs a version of the two-player protocol (from Section~\ref{sec:r-2party}) from $A$'s perspective and where players $P_2, \ldots, P_k$ serve jointly as the second player $B$.  To do so, we follow the distributed sampling approach outlined in Theorem \ref{thm:new-generickway}.  Specifically, we fix a parameter $c$ (set $c = 0.2$).  Each other node reports the total weight $w(D_i)$ of their data to $P_1$, who then reports back to each node what fraction of the total data $w(D_i)/w(D)$ they own.   Then each player sends the coordinator a random sample of size $s_{c,d} w(D_i)/w(D)$.  Recall that we require $s_{c,d} = O(d \log \log (1/\eps))$ in this case to account for probability of failure over all rounds.  
The union of these sets at $P_1$ satisfies the sampling condition in Lemma \ref{lem:cnet} for $\cup_{i=2}^k D_i$. 
$P_1$ computes a classifier on the union of its data and this joint sample and all previous joint samples, and sends the resulting classifier back to all the nodes.  Sending this classifier to each party requires $O(kd)$ words of communication.  
The process repeats for $T = \log_2(1/\eps)$ rounds.

\begin{theorem}
\label{thm:main-kparty}
The randomized $k$-party protocol for $\eps$-error linear separators in $\mathbb{R}^d$ terminates in $T = O(\log(1 / \eps))$ rounds using $O((kd + d^2 \log \log (1/\eps))\log(1/\eps))$ words of communication, and has a constant probability of failure.
\end{theorem}
\begin{proof}
The correctness and bound of $T = O(\log(1/\eps))$ rounds follows from Theorem \ref{thm:main-2party}, since, aside from the total weight gathering step, from party $P_1$'s perspective it appears to run the protocol with some party $B$ where $B$ represents parties $P_2, P_3, \ldots, P_k$.  
The communication for $P_1$ to collect the samples from all parties is $O(kd + d s_{c,d}) = O(kd + d^2 \log \log (1/\eps))$.  And it takes $O(dk)$ communication to return $h_A$ to all $k-1$ other players.  
Hence the total communication over $T = O(\log(1/\eps))$ rounds is $O((kd + d^2 \log \log(1/\eps)) \log(1/\eps))$ as claimed.  
\end{proof}

However, this randomized sampling algorithm required a sample of size $s_{c,d} = O(d \log \log (1/\eps))$, we can achieve a different communication trade-off using the deterministic construction.  We can no longer use the result from Theorem \ref{thm:new-generickway} since that has a probability of failure.   In this case, in each round each party $P_i$ communicates a deterministically constructed set $S_{c,i}$ of size $s_{c,d} = O(d)$, then the coordinator $P_1$  computes a classifier that correctly classifies points from all of these sets, and hence has at most $c w(D_i)$ weight of points misclassified in each $D_i$.  The error is at most $c w(D_i)$ on each dataset $D_i$, so the error on all sets is at most $c \sum_{i=2}^k w(D_i) = c w(D)$.  Again using $T = O(\log (1/\eps))$ rounds we can achieve the following result.  

\begin{theorem}
\label{thm:main-kparty-det}
The deterministic $k$-party protocol for $\eps$-error linear separators in $\mathbb{R}^d$ terminates in $T = O(\log(1 / \eps))$ rounds using $O(kd^2 \log(1/\eps))$ words of communication.
\end{theorem}



\section{Experiments}
\label{sec:experiments}
In this section, we present empirical results, using \ritsupp, for finding linear classifiers in $\mathbb{R}^d$ for two-party and $k$-party scenarios. We empirically compare amongst the following approaches.
\begin{itemize} \denselist
\item{\naiv:} a naive approach that sends all data from $(k-1)$ nodes to a coordinator node and then learns at the coordinator. For any dataset, this accuracy is the best possible.
\item{\vote:} a simple voting strategy that trains classifiers at each individual node and sends over the $(k-1)$ classifiers to a coordinator node. For any datapoint, the coordinator node predicts the label by taking a vote over all $k$ classifiers. 
\item{\randorg:} each of the $(k-1)$ nodes sends a random sample of size $s_{\eps,d}$ to a coordinator node and then a classifier is learned at the coordinator node using all of its own data and the samples received.
\item{\rand:} a cheaper version of \randorg that uses a random sample of size $9d$ from each party each round; this value was chosen to make this baseline technique as favorable as possible.  
\item{\supo:} \itsupp that selects informative points heuristically~\citep{daume12distributed}. A node is chosen as the coordinator and the coordinator exchanges maximum margin support points with each of the $(k-1)$ nodes. This continues until the training accuracy reaches within $(1-\eps)$ of the optimal (i.e., $(1-\eps)100\%$ in our case since we assume linearly separable data) or the communication cost equals the total size of the data at $(k-1)$ non-coordinator nodes (i.e., the cost for \naiv).
\item{\oalgo:} \ritsupp that randomly samples points based on the distribution of the weights and runs for $5 \log(1/\eps)$ number of rounds (ref. Section~\ref{sec:r-2party}).
\item{\oalgoes:} a cheaper version of \oalgo with an early stopping condition. The protocol is stopped early if the training accuracy has reached within $(1-\eps)$ of the optimal, i.e., $(1-\eps)100\%$.
\end{itemize}
We do not compare results with \med~\citep{daume12distributed} as it does not work on datasets beyond two dimensions. 
For all these methods, SVM (from libSVM~\citep{libsvm} library), with a linear kernel, was used as the underlying classifier. We report training accuracy and communication cost. The training accuracy is computed over the combined dataset $D$ with an $\eps$ value of $0.05$ (where applicable). The communication cost (in words) of all methods are reported as ratios with reference to \oalgoes as the base method. All numbers reported are averaged over $10$ runs of the experiments; standard deviations are reported where appropriate. For \oalgo and \oalgoes, we use $\rho = 0.75$.


\paragraph{Communication Cost Computation.}
In the following, we describe the communication cost computation for each method.  Each example point sent from one node to another incurs a communication cost of $d+1$, since it requires $d$ words to describe its position in $\mathbb{R}^d$ and $1$ word to describe its sign.  Similarly, each linear classifier requires $d+1$ words of communication to send; $d$ words to describe its direction, and $1$ word to describe its offset.  
\begin{itemize} \denselist
\item{\naiv:} assuming node $1$ to be coordinator, the total cost is the number of words sent over by each node to the coordinator and is equal to $\sum_{i=2}^k (d+1)|D_i|$. 
\item{\vote:} each node sends over its classifier to the coordinator node which incurs a total cost of $(d+1)(k-1)$. 
\item{\randorg:} the cost is equal to $(k-1)(d+1)s_{\eps,d} = (k-1)(d+1)(d/\eps) \log (d/\eps)$ times some constant where we set the constant to $1$. 
\item{\rand:} despite the theoretical cost of $(k-1)(d+1)s_{\eps,d} = (k-1)(d+1)(d/\eps) \log (d/\eps)$ (same as \randorg), in practice the random sampling based approach performs well with far fewer samples. Starting with a sample size of $5$, we first perform a doubling search to find the range within which \rand achieves $\eps$-optimal accuracy and then do binary search within this range to pick the smallest value for the sample size. Our goal is to pick one value that performs well across all of our datasets. In our case, $9d$ seems to work well for all the datasets we tested. Thus, in our case, \rand incurs a total cost of $9d(d+1)(k-1)$ words.
\item{\supo:} let $SP_i$ denote the support set of node $i$. Assuming node $1$ to be coordinator, the total cost in each round is equal to $(d+1)(k-1)|SP_1| + \sum_{i=2}^k (d+1)|SP_i|$ (the number words sent by the coordinator to all $(k-1)$ nodes plus the number of words sent back by the $(k-1)$ nodes to the coordinator). The cost accumulates over rounds until the target accuracy is reached or until the cost equals the total size of the data at $(k-1)$ non-coordinator nodes (i.e., the cost for \naiv). 
\item{\oalgo:} for our algorithm the cost incurred in each round is $(d+1) s_{c,d} (k-1) + (d+1) (k-1)$ words.  The first term comes from each player other than the coordinator sending $s_{c,d}$ points to the coordinator. The second term accounts for the coordinator replying with a classifier to each of those $(k-1)$ other players. However, we observe that exchanging a small constant number of samples, instead of $s_{c,d}$, each round works quite well in practice for all of our datasets.  For our analysis we had set $c = 1/5$ indicating that $s_{c,d}$ is some constant times $25 d$.  But in our experiments, we use a much smaller sample size of $100$ per round, with a word cost of $100(d+1)$ per round. The search process to find this smaller sample size is the same as described in \rand. The number of rounds for \oalgo is $5\log_2{(1/0.05)}\log_2{\log_2{(1/0.05)}} \simeq 5 \times 10 = 50$.
\item{\oalgoes:} similar to \oalgo, the sample size chosen in $100$ and the cost is $100(d+1)(k-1) + (d+1)(k-1)$ words times the number of rounds until the early stopping criterion is met.
\end{itemize}

Note that given our cost computation, for some datasets the cost of \randorg, \rand and \oalgo can exceed the cost of \naiv (see, for example, \cancer). For those datasets, the size of the data is small compared to the dimensions. As a result, the communication costs (in number of points) for (a)  \randorg: $(k-1)s_{\eps,d} = (k-1)(d/\eps) \log (d/\eps)$, (b) \rand: $9d(k-1)$, and (c) \oalgo: $(100(k-1)+(k-1))T=101(k-1) \times 50$ are large compared to the total size of the data at the $(k-1)$ non-coordinator nodes (i.e., the cost of \naiv). 

\paragraph{Datasets.}
We report results for two-party and four-party protocols on both synthetic and real-world datasets. 

Six datasets, three each for two-party and four-party case, have been generated synthetically from mixture of Gaussians. Each Gaussian has been carefully seeded to generate different data partitions. For \done, \dtwo, \dfour, \dfive, each node contains $5000$ data points ($2500$ positive and $2500$ negative) whereas for \dthr and \dsix, each node contains $8500$ data points ($4250$ positive and $4250$ negative) and all of these datapoints lie in $50$ dimensions. Additionally, we investigate the performance of our protocols on three real-world UCI datasets~\cite{UCIrepo}. Our goal is to select datasets that are linearly separable or almost linearly separable. We choose \cancer and \mush from the LibSVM data repository~\citep{libsvm}. 

The proposed protocol works for perfectly separable datasets. However, this assumption is too idealistic and in practice real-world datasets are seldom perfectly separable either because of presence of noise or due to limitations of linear classifiers (for example, what if the data has a non-linear decision boundary). So most of datasets have some amount of noise in them. This also shows that although our protocols were designed for noiseless data then work well on noisy datasets too. However, when applied on noisy data, we do not guarantee the communication bounds that were claimed for noiseless datasets. 

For the datasets that are not perfectly separable, the accuracy of \naiv (with some tolerance) that learns an SVM on the entire data can be considered to be the best accuracy that can be achieved for that particular dataset. Table~\ref{tab:datasets} presents a summary of the datasets, the best possible accuracy that can be achieved and also the accuracy required to yield an $\eps$-optimal classifier with $\eps = 0.05$.

Finally, in Tables~\ref{tab:syn-2party-results}-\ref{tab:real-results}, we highlight (in bold) the protocol that achieves the required accuracy and the lowest communication cost and thus is the best among the methods compared. By best we mean that the method has the cheapest communication cost as well an accuracy that is more that $(1-\eps)$ times the optimal, i.e., $95\%$ for our case for $\eps=0.05$. As will be frequently seen for \vote, the communication cost is the cheapest but the accuracy is far from the desired $\eps$-error specified, and in such circumstances we do not deem \vote as the best method.

\begin{table*}[!t]
	{\small
	\centering
	\begin{tabular}{|c|c|c|c|c|c|c|c|c|c|} 
		\hline
		{\bf Dataset} & {\bf total \#}  & \multicolumn{2}{|c|}{\bf \# of points per player} &{\bf dimensions}  &  {\bf type}  & {\bf perfectly} & {\bf best} & {\bf $\eps$-optimal}\\
		\cline{3-4}
		              &  {\bf of points} &  2-player  &  4-player  &         &            & {\bf separable?} & {\bf accuracy} & {\bf accuracy} \\
		\hline                                                                                           
		\hline
		\done         & 10000   &  5000      &   -        &  50     & synthetic  &   no  & 99.23 &   95.00    \\              
		\dtwo         & 10000   &  5000      &   -        &  50     & synthetic  &   no  & 97.91 &   95.00    \\              
		\dthr         & 17000   &  8500      &   -        &  50     & synthetic  &   no  & 97.39 &   95.00    \\              
		\hline
		\dfour        & 20000   &    -       &  5000      &  50     & synthetic  &   no  & 99.26 &   95.00    \\              
		\dfive        & 20000   &    -       &  5000      &  50     & synthetic  &   no  & 97.97 &   95.00    \\              
		\dsix         & 34000   &    -       &  8500      &  50     & synthetic  &   no  & 97.47 &   95.00    \\              
		\hline
    \cancer  	 	  &  683    &   342      &   171      &   10    & real		   &   no  & 97.07 &   95.00    \\              
		\mush         &  8124   &  4062      &  2031      &  112    & real       &   yes &  100  &   95.00    \\              
		%
		\hline
	\end{tabular}
	\caption{Summary of datasets used ($\eps = 0.05$).}
	\label{tab:datasets}
	}
\end{table*}

\subsection{Synthetic Results}
\label{ssec:syn-res}

Table~\ref{tab:syn-2party-results} compares the performance metrics of the aforementioned protocols for \emph{two}-parties. As can be seen, \vote performs the best for \done and \rand performs the best for \dtwo. For \dthr, \oalgoes requires the least amount of communication to learn an $\eps$-optimal distributed classifier. 
Note that, for \dtwo and \dthr, both \vote and \supo fail to produce a $\eps$-optimal ($\eps=0.05$) classifier. \supo exhibits this behavior despite incurring a communication cost that is as high as \naiv. Note that the cost of \supo being the same as \naiv does not imply that \supo send overs all points. Rather the accumulated cost of the support points become the same as the cost of \naiv at which point we stop the algorithm. Usually, by this point, the accuracy of \supo saturates and does not improve with exchange of more support points.

\begin{table*}[!t]
	\centering
	\begin{tabular}{|c||c|c||c|c||c|c||} 
		\hline
                 & \multicolumn{2}{c||}{\bf\done}      & \multicolumn{2}{|c||}{\bf\dtwo}     & \multicolumn{2}{|c||}{\bf\dthr}   \\
		\cline{2-7}                                                             
       			     & {\bf Acc}          & {\bf Cost}     & {\bf Acc}         & {\bf Cost}      & {\bf Acc} 	         & {\bf Cost}  \\
		\hline        			                                                                                                                                 
		\naiv        & 99.23 (0.0)        & 49.02          & 97.91 (0.0)       & 6.18           &  97.39 (0.0)        & 19.08       \\
		\vote        & {\bf95.00 (0.0)}   &  {\bf0.01}     & 60.64 (0.0)       & 0.01           &  74.55 (0.0)        &  0.01       \\
		\randorg     & 99.02 (0.0)        & 29.41          & 97.72 (0.0)       & 3.71           &  97.16 (0.0)        &  6.74       \\
		\rand        & 96.64 (0.1)        &  4.41          & {\bf95.13 (0.1)}  & {\bf0.56}      &  96.03 (0.1)        &  1.01       \\
		\supo        & 96.39 (0.0)        &  4.26          & 93.76 (0.0)       & 6.18           &  73.62 (0.0)        & 19.08       \\
    \oalgo       & 98.66 (0.1)        & 49.51          & 97.59 (0.1)       & 6.24           &  97.11 (0.1)        & 11.34       \\
    \oalgoes     & 95.00 (0.0)        &  1.00          & 95.17 (0.1)       & 1.00           &  {\bf95.25 (0.2)}   &  {\bf1.00}  \\
		\hline
	\end{tabular}
	\caption{Mean accuracy (Acc) and communication cost (Cost) required by \emph{two-party} protocols for synthetic datasets.}
	\label{tab:syn-2party-results}
\end{table*}

\begin{table*}[!t]
	\centering
	\begin{tabular}{|c||c|c||c|c||c|c||} 
		\hline
                 & \multicolumn{2}{c||}{\bf\dfour}     & \multicolumn{2}{|c||}{\bf\dfive}    & \multicolumn{2}{|c||}{\bf\dsix} \\
		\cline{2-7}                                                  
                 & {\bf Acc}         & {\bf Cost}      & {\bf Acc}        & {\bf Cost}       & {\bf Acc} 	         & {\bf Cost} \\
		\hline 
		\naiv        & 99.26 (0.0)       &100.00           & 97.97 (0.0)      &  12.72           & 97.47 (0.0)         &  54.84    \\
		\vote        & {\bf95.00 (0.0)}  &{\bf0.01}        & 65.83 (0.0)      &   0.01           & 75.52 (0.0)         &   0.01    \\
		\randorg     & 99.18 (0.0)       & 60.00           & 97.83 (0.0)      &   7.63           & 97.39 (0.0)         &  19.35    \\
		\rand        & 97.33 (0.1)       &  9.00           & 96.61 (0.1)      &   1.15           & 96.67 (0.1)         &   2.90    \\
		\supo        & 95.95 (0.0)       &  0.82           & 93.94 (0.0)      &  15.15           & 75.05 (0.0)         &  80.19    \\
    \oalgo       & 98.03 (0.2)       &  34.78  	       & 97.30 (0.1)      &   4.45           & 96.87 (0.1)         &  11.24    \\
    \oalgoes     & 95.11 (0.3)       &  1.00  	       & {\bf95.11 (0.2)} &   {\bf1.00}      & {\bf95.45 (0.2)}    &   {\bf1.00}    \\
		\hline
	\end{tabular}
	\caption{Mean accuracy (Acc) and communication cost (Cost) required by \emph{four-party} protocols for synthetic datasets.}
	\label{tab:syn-kparty-results}
\end{table*}

As shown in Table~\ref{tab:syn-kparty-results}, most of the two-party results carry over to the multiparty case. \vote is the best for \dfour whereas \oalgoes is the best for \dfive and \dsix. As earlier, both \vote and \supo do not yield an $0.05$-optimal distributed classifiers for \dfive and \dsix.

Figure~\ref{fig:syn-2party} (for two-party using \done) shows the communication costs (in \emph{log-scale}) with variations in the number of data points per node and the dimension of the data. Note that we do not report the numbers for \supo since \supo takes a long time to finish. However, for \done the numbers for \supo are similar to those of \rand and so their curves in the figure are also the same. Note that in Figure~\ref{fig:syn3-2party-dim}, the cost of \naiv increases as the number of dimensions increase. This is because the cost is multiplied by a factor of $(d+1)$, when expressed in words.

\begin{figure*}[!htbp]
\vspace{-0.45cm}
\centering
\subfigure[Communication cost vs Size]{
\includegraphics[width=9.00cm]{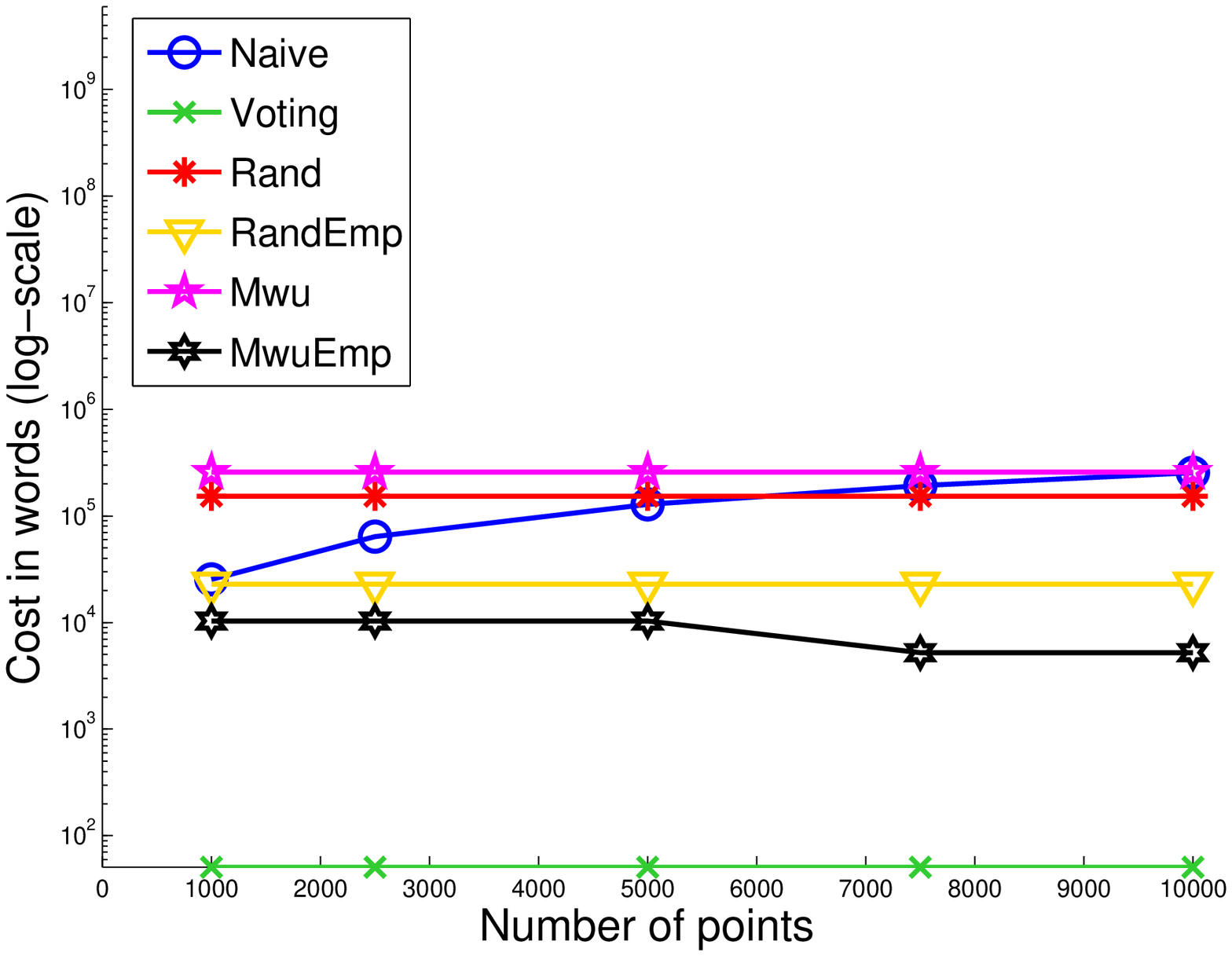}
\label{fig:syn3-2party-size}
\vspace{-.2in}
}
\hspace{-1.25cm}
\subfigure[Communication cost vs Dimension]{
\includegraphics[width=9.00cm]{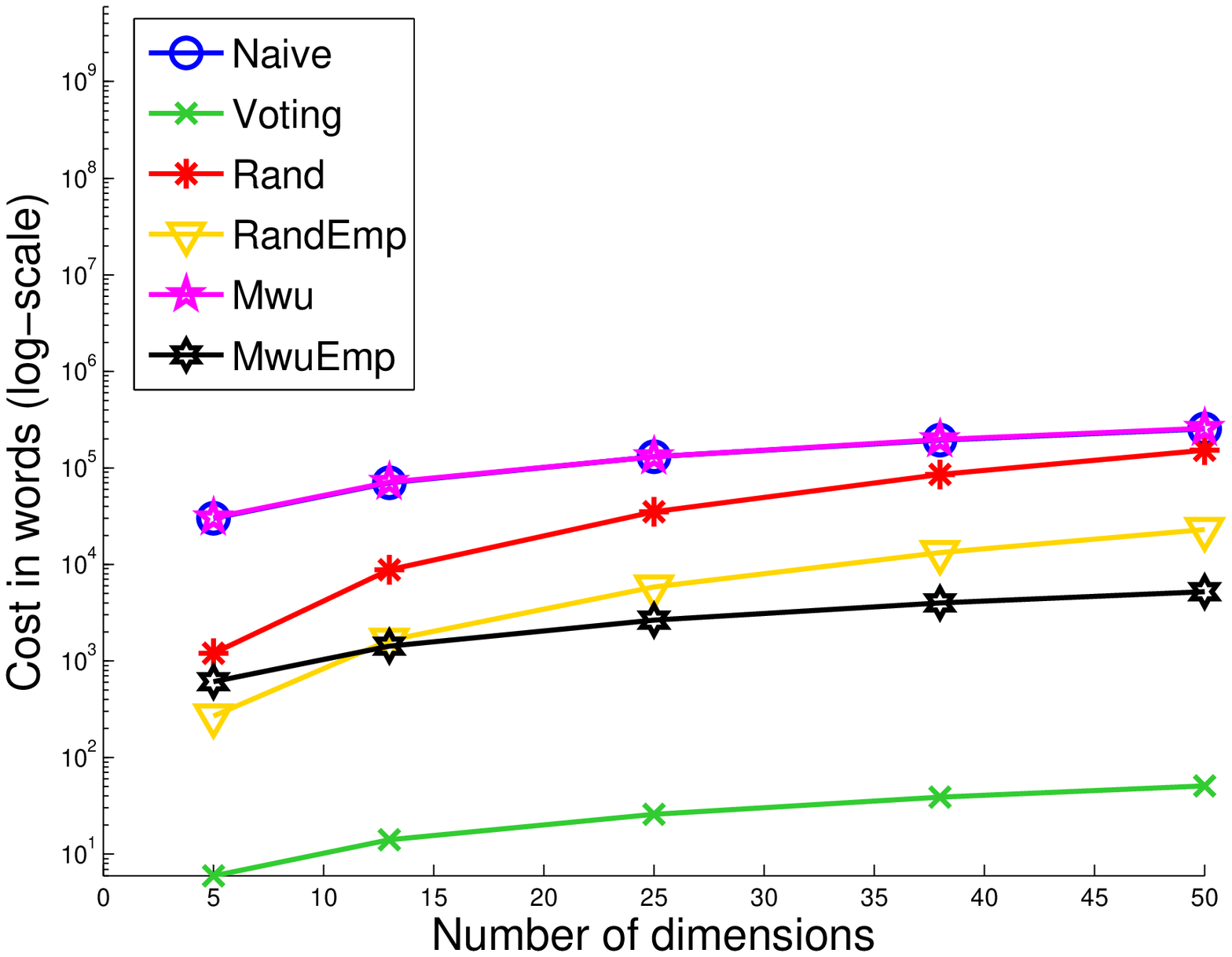}
\label{fig:syn3-2party-dim}
\vspace{-.2in}
}
\vspace{-.1in}
\caption{Communication cost vs Size and Dimensionality for \done with $2$-party protocol.}
\label{fig:syn-2party}
\vspace{-0.25cm}
\end{figure*}


\subsection{Real-World Data}
\label{ssec:real-res}
Table~\ref{tab:real-results} presents results for two-party protocols and four-party protocols using real-world datasets. Other that the two-party case for \mush, \vote performs the best in all other case. However, note that for \mush using two-party protocol, \vote does not yield a $0.05$-optimal distributed classifier.

\begin{table*}[!htbp]
	\centering
	\begin{tabular}{|c||c|c||c|c||} 
		\hline
		             & \multicolumn{2}{c||}{\bf\cancer}      & \multicolumn{2}{|c||}{\bf\mush}         \\
		\cline{2-5}                                        
                 &     \bf Acc     &     \bf Cost      &     \bf Acc 	   &     \bf Cost 	       \\
		\hline                                             
		\multicolumn{5}{|c||}{2-party} \\                  
		\hline                                             
		\naiv        & 97.07 (0.0)     &  3.34             & 100.00 (0.0)    &  20.01                \\
		\vote        & {\bf97.36 (0.0)}&{\bf0.01}          & 88.38 (0.0)     &   0.00                \\
		\randorg     & 97.16 (0.1)     &  4.52             & 100.00 (1.1)    &  36.97                \\
		\rand        & 96.90 (0.2)     &  0.88             & 100.00 (0.0)    &   4.97                \\
		\supo   	   & 96.78 (0.0)     &  0.22             & 100.00 (0.0)    &   1.11                \\
    \oalgo       & 97.36 (0.2)     & 49.51             & 100.00 (0.0)    &  24.88                \\
    \oalgoes     & 96.87 (0.4)     &  1.00             & {\bf99.73 (0.5)}&  {\bf1.00}            \\
		\hline                                             
		\multicolumn{5}{|c||}{4-party} \\                  
		\hline                                             
	  \naiv        & 97.07 (0.0)     &  1.00             & 100.00 (0.0)    &  28.61                 \\
	  \vote        & {\bf97.36 (0.0)}&  {\bf0.03}        & {\bf95.67 (0.0)}&  {\bf0.01}            \\
		\randorg     & 97.19 (0.1)     &  12.81            & 100.00 (0.6)    & 105.70                \\
	  \rand        & 96.99 (0.1)     &  2.50             & 99.99 (0.0)     &  14.20                 \\
	  \supo        & 96.78 (0.0)     &  0.56             & 100.00 (0.0)    &   2.34                 \\
	  \oalgo       & 97.00 (0.2)     &  48.46            & 100.00 (0.1)    &  24.65                \\
	  \oalgoes     & 96.97 (0.3)     &  1.00             & 98.86 (0.4)     &   1.00                \\
		\hline
	\end{tabular}
	\caption{Mean accuracy (Acc) and communication cost (Cost) required by all protocols for real-world datasets.}
	\label{tab:real-results}
\end{table*}

The results for communication cost (in \emph{log-scale}) versus data size and communication cost (in \emph{log-scale}) versus dimensionality are provided in Figure~\ref{fig:real-2party} for two-party protocol using the \mush dataset. \oalgoes (denoted by the black line) is comparable to \supo and cheaper than all other baselines (except \vote).


\begin{figure*}[!htbp]
\vspace{-0.45cm}
\centering
\subfigure[Communication cost vs Size]{
\includegraphics[width=9.00cm]{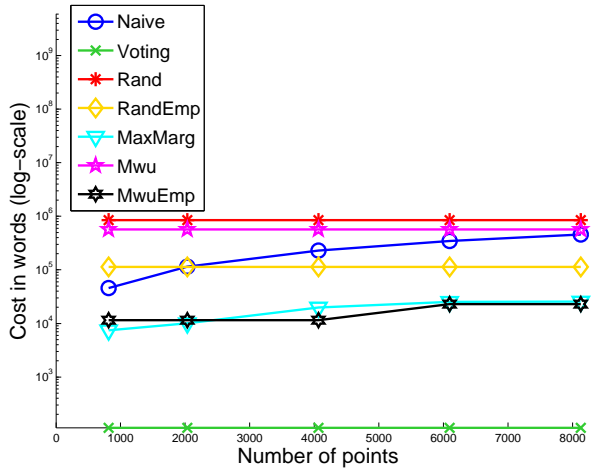}
\label{fig:adult-2party-size}
\vspace{-.2in}
}
\hspace{-1.25cm}
\subfigure[Communication cost vs Dimension]{
\includegraphics[width=9.00cm]{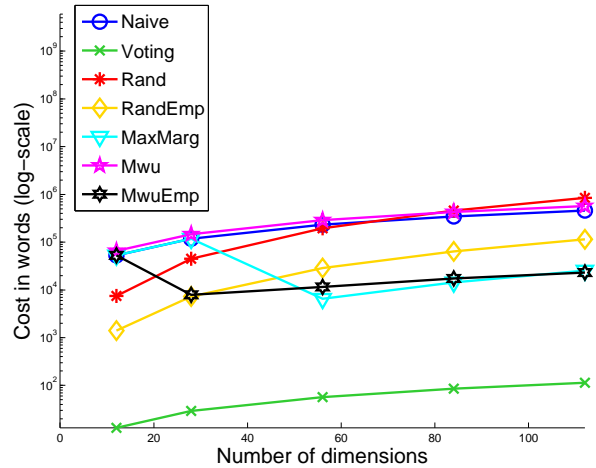}
\label{fig:adult-2party-dim}
\vspace{-.2in}
}
\vspace{-.1in}
\caption{Communication cost vs Size and Dimensionality for \mush with $2$-party protocol.}
\label{fig:real-2party}
\vspace{-0.25cm}
\end{figure*}


\paragraph{Remarks.}
The goal of our experiments is to show that our protocols perform well, particularly for difficult or adversarially partitioned datasets. For easy datasets, any baseline technique can perform well. Indeed, \vote performs the best on \done and \dfour and \rand performs better than others on \dtwo. For the remaining three cases on synthetic datasets, \oalgoes outperforms the other baselines. On real world data, \vote usually performs well. However, as we have shown earlier, for some datasets \vote and \supo fail to yield an $\eps$-optimal classifier. In particular for \mush, using the two-party protocol, the accuracy achieved by \vote is far from $\eps$-optimal. This and earlier results show that there exists scenarios where \vote and \supo perform particularly worse and so learning by majority voting or by exchanging support points in between nodes is not a good strategy in distributed settings, even more so when the data is partitioned adversarially.

\section{Distributed Optimization}
\label{sec:opt}

Many learning problems can be formulated as convex (or even linear or semidefinite) optimizations~\citep{Bennett:2006:IOM:1248547.1248593}. In these problems, the data (points) act as constraints to the resulting optimization; for example, in a standard SVM formulation, there is one constraint for each point in the training set. 

Since in our distributed setting, points are divided among the different players, a natural \emph{distributed optimization} problem can be stated as follows. Each player $i$ has a set of constraints $C_i = \{ f_{ij}(x) \ge 0 \}$,  and the goal is to solve the optimization $\min g(x)$ subject to the union of constraints $\cup_i C_i$. As earlier, our goal is to solve the above with minimum communication.

A general solution for communication-efficient distributed convex optimization will allow us to reduce communication overhead for a number of distributed learning problems. In this section, we illustrate two algorithm design paradigms that achieves this for distributed convex optimization.

\subsection{Optimization via Multi-Pass Streaming}
\label{ssec:streaming}

A \emph{streaming algorithm}~\citep{muthubook2005data} takes as input a sequence of items $x_1, \ldots x_n$. The algorithm is allowed working space that is \emph{sublinear in $n$}, and is only allowed to look at each item once as it \emph{streams} past. A \emph{multipass} streaming algorithm is one in which the algorithm may make more than one pass over the data, but is still limited to sublinear working space and a single look at each item in each pass. 

The following lemma shows how any (multipass) streaming algorithm can be used to build a multiparty distributed protocol. 

\begin{lemma}
\label{stream-dist}
  Suppose that we can solve a given problem $P$ using a streaming algorithm that has $s$ words of working storage and makes $r$ passes over the data. Then there is a $k$-player distributed algorithm for $P$ that uses $krs$ words of communication. 
\end{lemma}

Before proving the above lemma, we note that streaming algorithms often have $s= O(\text{poly}\log n)$ and $r = O(\log n)$, indicating that the total communication is $O(k\text{\ poly}\log n)$ words, which is sublinear in the input size. 

\begin{proof}
  For ease of exposition, let us first consider the case when $k=2$. Consider a streaming algorithm $S$ satisfying the conditions above. The simulation works by letting the first player $A$ simulate the first half of $S$, and letting the second player $B$ simulate the second half. Specifically, the first player $A$ simulates the behavior of $S$ on its input. When this simulation of $S$ exhausts the input at $A$, $A$ sends over the contents of the working store of $S$ to $B$. $B$ restarts $S$ on its input using this working store as $S$'s current state. When $B$ has finished simulating $S$ on its input, it sends the contents of the working storage back to $A$. This completes one pass of $S$, and used $s$ words of communication. The process continues for $r$ passes. 

If there are $k$ players $A_1, \ldots, A_k$ instead of two, then we fix an arbitrary ordering of the players. The first player simulates $S$ on its input, and at completion passes the contents of the working store to the next one, and so on. Each pass now requires $O(ks)$ words of communication, and the result follows. 
\end{proof}

We can apply this lemma to get a streaming algorithm for fixed-dimensional linear programming\footnote{Fixed-dimensional linear programming is the case of linear programming where the dimension is \emph{not} part of the input. Effectively, this means that exponential dependence on the dimension is permitted; the dependence on the number of constraints remains polynomial as usual.}. This relies on an existing result~\citep{DBLP:journals/dcg/ChanC07}:

\begin{theorem}[\citep{DBLP:journals/dcg/ChanC07}]
Given $n$ halfspaces in $\reals^d$ (for $d$ constant), we can compute the lowest point in their intersection by a $O(1/\delta^{d-1})$-pass Las Vegas algorithm that uses $O((1/\delta^{O(1)})n^\delta)$ space and runs in time $O((1/\delta^{O(1)})n^{1+\delta})$ with high probability, for any constant $\delta > 0$. 
\end{theorem}

\begin{corollary}
\label{coro-lp}
  There is a $k$-player algorithm for solving distributed linear programming that uses $O(k (1/\delta^{d+O(1)})n^\delta)$ communication, for any constant $\delta >0$. 
\end{corollary}

While the above streaming algorithm can be applied as a blackbox in Corollary~\ref{coro-lp}, looking deeper into the streaming algorithm reveals room for improvement. As in the case of classification, suppose that we are permitted to violate an $\eps$-fraction of the constraints. It turns out that the above streaming algorithm achieves its bounds by eliminating a fixed fraction of constraints in each space, and thus requires $\log_r n$ passes, where $r = n^{\Theta(\delta)}$. If we are allowed to violate an $\eps$-fraction of constraints, we need only run the algorithm for $\log_r 1/\eps$ passes, where $r$ is now $O(1/\eps^{\Theta(\delta)})$. This allows us to replace $n$ in all terms by $1/\eps$, resulting in an algorithm with communication \emph{independent of $n$}.  

\begin{corollary}
\label{coro-lp-apx}
  There is a $k$-player algorithm for solving distributed linear programming that violates at most an $\eps$-fraction of the constraints, and that uses $O(k (1/\delta^{d+O(1)})(1/\eps)^\delta)$ communication, for any constant $\delta > 0$. 
\end{corollary}

\subsection{Optimization via Multiplicative Weight Updates}
\label{ssec:mwu}

The above result gives an approach for solving \emph{fixed-dimensional} linear programming (exactly or with at most $\eps n$ violated constraints) in a distributed setting. There is no known streaming algorithm for arbitrary-dimensional linear programming, so the stream-algorithm-based design strategy cannot be used. However we will now show that the multiplicative weight update method can be applied in a distributed manner, and this allows us to solve general linear programming problems, as well as SDPs and other convex optimizations. 

We first consider the problem of solving a general LP of the form $\min g^\top x$, subject to $Ax \ge b$, $x \in P$, where $P$ is a set of ``soft'' constraints (for example, $x \ge 0$) and $Ax \ge b$ are the ``hard'' constraints. Let $z^* = \min g^\top x^*$ be the optimal value of the LP, obtained at $x^*$. Then the multiplicative weight update method can be used to obtain a solution $\tilde{x}$ such that $z^* = g^\top \tilde{x}$ and all (hard) constraints are satisfied approximately, i.e $\forall i$, $A_i \tilde{x} \ge b_i - \eps$, where $A_i x \ge b_i$ is one row of the constraint matrix. We call such a solution a \emph{soft-$\eps$-approximation} (to distinguish it from a traditional approximation in which all constraints would be satisfied \emph{exactly} and the objective would be approximately achieved. 

The standard protocol works as follows~\citep{Arora05MWUsurvey}. We assume that the optimal $z^*$ has been guessed (this can be determined by binary search), and define the set of ``soft'' constraints to be $\EuScript{P} = P \cup \{ x \mid g^\top x = z^*\}$. Typically, it is easy to check for feasibility in $\EuScript{P}$. We define a \emph{width} parameter $\rho = \max \{\max_{i \in [n],x \in \EuScript{P}} A_i x - b_i ,1\}$. Initialize $m_i(0) = 0$. Then we run $T= O(\rho^2\ln n/\eps^2)$ iterations (with $t = 1, 2, \ldots, T$) of the following:
\begin{enumerate} \denselist
  \item Set $p_i(t) = \exp(- \eps m_i(t-1)/2)$.
  \item Find feasible $x(t)$ in $\EuScript{P} \cup \{ x \mid \sum_i p_i A_i x \ge \sum_i p_i b_i\}$.
  \item $m_i(t) = m_i(t-1) + A_i x(t) - b_i$. 
\end{enumerate}
At the end, we return $\overline{x} = (1/t)\sum_t x(t)$ as our soft-$\eps$-approximation for the LP.

We now describe a two-party distributed protocol for linear programming adapted from this scheme. The protocol is asymmetric. Player $A$ finds feasible values of $x$ and player $B$ maintains the weights $m_i$. Specifically, player $A$ constructs a feasible set $\EuScript{P}$ consisting of the original feasible set $P$ and all of its own constraints. As above, $B$ initializes a weight vector $m$ to all zeros, and then sends over the \emph{single} constraint $\sum_i p_i A_i x \ge \sum_i p_i b_i$ to $A$. Player $A$ then finds a feasible $x$ using this constraint as well as $\EuScript{P}$ (solving a linear program) and then sends the resulting $x$ back to $B$, who updates its weight vector $m$. 

Each round of communication requires $O(d)$ words of information, and there are $O(\rho^2\ln n/\eps^2)$ rounds of communication. Notice that this is exponentially better than merely sending over all constraints. 

\begin{theorem}
  There is a 2-player distributed protocol that uses $O(d\rho^2\ln n/\eps^2)$ words of communication to compute a soft-$\eps$-approximation for a linear program. 
\end{theorem}

A similar result applies for semidefinite programming (based on an existing primal MWU-based SDP algorithm~\citep{DBLP:conf/focs/AroraHK05}) as well as other optimizations for which the MWU applies, such as rank minimization~\citep{DBLP:conf/icml/MekaJCD08}, etc. 


\section{Conclusion}
\label{sec:disc}
In this work, we have proposed a simple and efficient protocol that learns an $\eps$-optimal distributed classifier for hyperplanes in arbitrary dimensions. The protocol also gracefully extends to $k$-players. Our proposed technique \ritsupp relates to the MWU-based meta framework and we exploit this connection to extend \ritsupp for distributed convex optimization problems. This makes our protocol applicable to a wide variety of distributed learning problems that can be formulated as an optimization task over multiple distributed nodes.

\bibliographystyle{icml2012}
\bibliography{ml}

\end{document}